\newif\ifJOURNAL
\JOURNALfalse
\newif\ifCONF    % used for the COLT 2015 Open Problems submission
\CONFfalse
\newif\ifarXiv
\arXivfalse
\newif\ifWP
\WPfalse
\newif\ifFULL
\FULLfalse
\newif\ifPURE
\PUREfalse

\arXivtrue
\newif\ifnotCONF  % derivative conditional
\notCONFtrue
\ifCONF\notCONFfalse\fi

\newif\ifnotarXiv  % derivative conditional
\notarXivtrue
\ifarXiv\notarXivfalse\fi

\newif\ifTR  % derivative conditionals (TR = arXiv or WP)
\TRfalse
\ifarXiv\TRtrue\fi
\ifWP\TRtrue\fi
\newif\ifnotTR
\notTRtrue
\ifarXiv\notTRfalse\fi
\ifWP\notTRfalse\fi

\newif\ifCONFnotPURE  % derivative conditional (CONF but not PURE)
\CONFnotPUREtrue
\ifnotCONF\CONFnotPUREfalse\fi
\ifPURE\CONFnotPUREfalse\fi

\ifCONFnotPURE    % For the Gurevich Festschrift.
  \documentclass[runningheads,a4paper]{llncs}
  \usepackage{amsmath,amsfonts,amssymb,latexsym,graphicx,enumitem}
\fi

\ifPURE
  \documentclass{article}
  \usepackage{amsmath,amsthm,amsfonts,amssymb,latexsym,graphicx,enumitem}

\fi

\ifarXiv
  \documentclass[10pt]{article}
  \usepackage{amsmath,amsthm,amsfonts,amssymb,latexsym,graphicx}
  \title{The fundamental nature of the log loss function}
  \author{Vladimir Vovk\\
  \texttt{v.vovk{\rm@}rhul.ac.uk}}
\fi

\ifWP
  \documentclass[10pt]{article}
  \usepackage{amsmath,amsthm,amsfonts,amssymb,latexsym,graphicx}
  \input{OT2enc.def}
  
  \usepackage{CJK}
  \input{/Doc/Computing/Latex/kp.txt}
  \title{The Fundamental Nature of the Log Loss Function}
  \author{Vladimir Vovk}
  
\fi

\allowdisplaybreaks[4]

\DeclareMathOperator{\Expect}{\mathbb{E}} % expectation
\DeclareMathOperator{\Loss}{Loss}         % loss
\newcommand{\K}{\mathcal{K}}        % universal loss

\newcommand{\Extra}[1]{}
\ifFULL
  \usepackage{color}
  \renewcommand{\Extra}[1]{\red{#1}}
  \newcommand{\red}[1]{\textcolor{red}{#1}}
  
  \newcommand{\bluebegin}{\begingroup\color{blue}}
  \newcommand{\blueend}{\endgroup}

\fi

\ifCONFnotPURE
\begin{document}
  \mainmatter
  \title{The Fundamental Nature of the Log Loss Function}
  \titlerunning{Log Loss}
  \author{Vladimir Vovk}
  % \thanks{Not used now.}
  \authorrunning{V.~Vovk}
  \institute{Department of Computer Science, Royal Holloway, University of London,\\
    Egham, Surrey, UK}
  \toctitle{The Fundamental Nature of the Log Loss Function}
  \tocauthor{Vladimir Vovk}
\fi

\ifPURE
\begin{document}
  \title{The Fundamental Nature of the Log Loss Function}
  \author{Vladimir Vovk\\
    Department of Computer Science, Royal Holloway, University of London,\\
    Egham, Surrey, UK}
\fi

\ifnotCONF
  \newtheorem{lemma}{Lemma}
  \newtheorem{proposition}{Proposition}
  \newtheorem{corollary}{Corollary}
  \newtheorem{theorem}{Theorem}
  \theoremstyle{definition}
  \newtheorem*{remark}{Remark}
  \begin{document}
\fi

\ifPURE
  \newtheorem{lemma}{Lemma}
  \newtheorem{proposition}{Proposition}
  \newtheorem{corollary}{Corollary}
  \newtheorem{theorem}{Theorem}
  \theoremstyle{definition}
  \newtheorem*{remark}{Remark}
\fi

\newenvironment{Question}
  {\trivlist\item[]\itshape}
  {\endtrivlist}

\maketitle

\begin{abstract}
  The standard loss functions used in the literature on probabilistic prediction
  are the log loss function, the Brier loss function, and the spherical loss function;
  however, any computable proper loss function can be used for comparison of prediction algorithms.
  This note shows that the log loss function is most selective
  in that any prediction algorithm that is optimal for a given data sequence
  (in the sense of the algorithmic theory of randomness) under the log loss function
  will be optimal under any computable proper mixable loss function;
  % I will state open problems around this phenomenon.
  on the other hand, there is a data sequence and a prediction algorithm
  that is optimal for that sequence under either of the two other standard loss functions
  but not under the log loss function.
\ifCONF
  \begin{keywords}
    algorithmic theory of randomness,
    mixability, predictive complexity, predictive randomness,
    probabilistic prediction, proper loss functions.
    % Submitted to EasyChair:
    % algorithmic theory of randomness,
    % Martin-L\"of tests,
    % Levin randomness,
    % predictive complexity,
    % Aggregating Algorithm.
  \end{keywords}
\fi
\end{abstract}

\section{Introduction}

\ifCONF
  In his work Yuri Gurevich has emphasized practical aspects of algorithmic randomness.
  In particular, he called for creating a formal framework
  allowing us to judge whether observed events can be regarded as random
  or point to something dubious going on
  (see, e.g., the discussion of the lottery organizer's wife winning the main prize
  in \cite{Gurevich/Passmore:2012}).
  The beautiful classical theory of randomness started by Andrey Kolmogorov and Per Martin-L\"of
  has to be restricted in order to achieve this goal
  and avoid its inherent incomputabilities and asymptotics.

  This note tackles another practically-motivated question:
  what are the best loss functions for evaluating probabilistic prediction algorithms?
  Answering this question, however, requires extending rather than restricting
  the classical theory of algorithmic randomness.
\fi

In the empirical work on probabilistic prediction in machine learning
(see, e.g., \cite{Caruana/NM:2006})
the most standard loss functions are log loss and Brier loss,
and spherical loss is a viable alternative;
all these loss functions will be defined later in this note.
It is important to understand which of these three loss functions
is likely to lead to better prediction algorithms.
We formalize this question using a generalization of the notion of Kolmogorov complexity
called predictive complexity (see, e.g., \cite{Kalnishkan:2015};
it is defined in Section~\ref{sec:repetitive}).
Our answer is that the log loss function is likely to lead to better prediction algorithms
as it is more selective:
if a prediction algorithm is optimal under the log loss function,
it will be optimal under the Brier and spherical loss functions,
but the opposite implications are not true in general.

As we discuss at the end of Section~\ref{sec:repetitive},
the log loss function corresponds to the classical theory of randomness.
Therefore, our findings confirm once again the importance of the classical theory
and are not surprising at all from the point of view of that theory.
But from the point of view of experimental machine learning,
our recommendation to use the log loss function rather than Brier or spherical
is less trivial.

This note is, of course, not the first to argue that the log loss function is fundamental.
For example, David Dowe has argued for it since at least 2008
(\cite{Dowe:2008}, footnote~175;
see \cite{Dowe:2013}, Section~4.1, for further references).
% (2008, 2013)
% (he refers to I.~J.~Good)
Another paper supporting the use of the log loss function is Bickel's \cite{Bickel:2007}.
\ifFULL\bluebegin
  According to Bickel (2007):
  if one is using the output of a scoring rule to evaluate the ability of the assessor,
  logarithmic is the only scoring rule consistent with the likelihood principle or the use of Bayes factors
  to update the weights assigned to different experts or forecasting systems (Winkler 1969, 1996).
  Bickel (2007) himself supports the use of the log loss function
  (comparing it to the Brier and spherical loss functions).
\blueend\fi
% Kelly (horse racing) seems to be irrelevant.

% The source for this note:
% \cite{Vovk:2001};
% (with its predecessor published as \cite{Vovk:1997Brier});
% some of the statements in this note are stated in \cite{Vovk:2001}
% with a reference to [23], a CLRC technical report,
% but that report is lost (or has never been completed).

\section{Loss Functions}

We are interested in the problem of binary probabilistic prediction:
the task is to predict a binary label $y\in\{0,1\}$
with a number $p\in[0,1]$;
intuitively, $p$ is the predicted probability that $y=1$.
The quality of the prediction $p$ is measured by a \emph{loss function}
$\lambda:[0,1]\times\{0,1\}\to\mathbb{R}\cup\{+\infty\}$.
Intuitively, $\lambda(p,y)$ is the loss suffered by a prediction algorithm
that outputs a prediction $p$ while the actual label is $y$;
the value $+\infty$ (from now on abbreviated to $\infty$) is allowed.
Following \cite{Reid/etal:2014},
we will write $\lambda_y(p)$ in place of $\lambda(p,y)$,
and so identify $\lambda$ with the pair of functions $(\lambda_0,\lambda_1)$
where $\lambda_0:[0,1]\to\mathbb{R}\cup\{+\infty\}$ and $\lambda_1:[0,1]\to\mathbb{R}\cup\{+\infty\}$.
We will assume that $\lambda_0(0)=\lambda_1(1)=0$,
% this is a natural assumption but probably not needed,
% apart from the nonnegativity of \lambda
% (which is not really needed but is used in Kalnishkan:2015)
that the function $\lambda_0$ is increasing,
that the function $\lambda_1$ is decreasing,
and that $\lambda_y(p)<\infty$ unless $p\in\{0,1\}$.

A loss function $\lambda$ is called \emph{$\eta$-mixable} for $\eta\in(0,\infty)$
if the set
$$
  \left\{
    (u,v)\in[0,1]^2
  \mid
    \exists p\in[0,1]:
    u \le e^{-\eta\lambda(p,0)} \text{ and }
    v \le e^{-\eta\lambda(p,1)}
  \right\}
$$
is convex;
we say that $\lambda$ is \emph{mixable} if it is $\eta$-mixable
for some $\eta$.

A loss function $\lambda$ is called \emph{proper} if,
for all $p,q\in[0,1]$,
\begin{equation}\label{eq:proper}
  \Expect_p
  \lambda(p,\cdot)
  \le
  \Expect_p
  \lambda(q,\cdot),
\end{equation}
where $\Expect_pf:=pf(1)+(1-p)f(0)$ for $f:\{0,1\}\to\mathbb{R}$.
It is \emph{strictly proper}
if the inequality in~\eqref{eq:proper} is strict whenever $q\ne p$.

We will be only interested in computable loss functions
(the notion of computability is not defined formally in this note;
see, e.g., \cite{Kalnishkan:2015}).
We will refer to the loss functions satisfying the properties listed above
as \emph{CPM} (computable proper mixable) loss functions.

Besides, we will sometimes make the following \emph{smoothness assumptions}:
\begin{itemize}
\item
  $\lambda_0$ is infinitely differentiable over the interval $[0,1)$
  (the derivatives at $0$ being one-sided);
\item
  $\lambda_1$ is infinitely differentiable over the interval $(0,1]$
  (the derivatives at $1$ being one-sided);
\item
  for all $p\in(0,1)$, $(\lambda'_0(p),\lambda'_1(p))\ne0$.
\end{itemize}
We will refer to the loss functions satisfying all the properties listed above
as \emph{CPMS} (computable proper mixable smooth) loss functions.

\ifFULL\bluebegin
  Not all proper mixable loss functions are strictly proper:
  the prediction set can have corners.
\blueend\fi

\subsection*{Examples}

The most popular loss functions in machine learning are the \emph{log loss function}
\begin{equation*} % \label{eq:log-loss}
  \lambda_1(p)
  :=
  -\ln p,
  \quad
  \lambda_0(p)
  :=
  -\ln(1-p)
\end{equation*}
and the \emph{Brier loss function}
$$
  \lambda(p,y)
  :=
  (y-p)^2.
$$
Somewhat less popular is the \emph{spherical loss function}
\begin{equation*} % \label{eq:spherical-loss}
  \lambda_1(p)
  :=
  1 - \frac{p}{\sqrt{p^2+(1-p)^2}},
  \quad
  \lambda_0(p)
  :=
  1 - \frac{1-p}{\sqrt{p^2+(1-p)^2}}.
\end{equation*}
All three loss functions are mixable,
% see, e.g., \cite{Vovk:2001competitive}, Section~2,
% in the case of the first two
as we will see later.
They are also computable (obviously),
strictly proper (this can be checked by differentiation),
and satisfy the smoothness conditions (obviously).
Being computable and strictly proper,
these loss functions can be used to measure the quality of probabilistic predictions.

\subsection*{Mixability and Propriety}

Intuitively, propriety can be regarded as a way of parameterizing loss functions,
and we get it almost for free for mixable loss functions.
The essence of a loss function is its \emph{prediction set}
\begin{equation}\label{eq:prediction-set}
  \left\{
    (\lambda_0(p),\lambda_1(p))
  \mid
    p\in[0,1]
  \right\}.
\end{equation}
When given a prediction set,
we can parameterize it by defining $(\lambda_0(p),\lambda_1(p))$
to be the point $(x,y)$ of the prediction set at which $\inf_{(x,y)}(py+(1-p)x)$ is attained.
% (Show that it is indeed attained.)
This will give us a proper loss function.
And if the original loss function satisfies the smoothness conditions
(and so, intuitively, the prediction set does not have corners),
the new loss function will be strictly proper.

\section{Repetitive Predictions}
\label{sec:repetitive}

Starting from this section we consider the situation,
typical in machine learning,
where we repeatedly observe data $z_1,z_2,\ldots$
and each observation $z_t=(x_t,y_t)\in\mathbf{Z}=\mathbf{X}\times\{0,1\}$
consists of an \emph{object} $x_t\in\mathbf{X}$ and its \emph{label} $y_t\in\{0,1\}$.
Let us assume, for simplicity, that $\mathbf{X}$ is a finite set,
say a set of natural numbers.

A \emph{prediction algorithm} is a computable function
$F:\mathbf{Z}^*\times\mathbf{X}\to[0,1]$;
intuitively, given a data sequence $\sigma=(z_1,\ldots,z_T)$
and a new object $x$,
$F$ outputs a prediction $F(\sigma,x)$ for the label of $x$.
For any data sequence $\sigma=(z_1,\ldots,z_T)$ and loss function $\lambda$,
we define the cumulative loss that $F$ suffers on $\sigma$ as
$$
  \Loss_F^{\lambda}(\sigma)
  :=
  \sum_{t=1}^T
  \lambda(F(z_1,\ldots,z_{t-1},x_t),y_t)
$$
(where $z_t=(x_t,y_t)$ and $\infty+a$ is defined to be $\infty$ for any $a\in\mathbb{R}\cup\{\infty\}$).
Functions $\Loss_F^{\lambda}:\mathbf{Z}^*\to\mathbb{R}$
that can be defined this way for a given $\lambda$
are called \emph{loss processes under $\lambda$}.
In other words, $L:\mathbf{Z}^*\to\mathbb{R}$ is a loss process under $\lambda$
if and only if $L(\Box)=0$ (where $\Box$ is the empty sequence) and
\begin{equation}\label{eq:loss}
  \forall\sigma\in\mathbf{Z}^* \,
  \forall x\in\mathbf{X} \,
  \exists p\in[0,1] \,
  \forall y\in\{0,1\}:
  L(\sigma,x,y)
  =
  L(\sigma)
  +
  \lambda(p,y).
\end{equation}

A function $L:\mathbf{Z}^*\to\mathbb{R}$ is said to be a \emph{superloss process under $\lambda$}
if \eqref{eq:loss} holds with $\ge$ in place of $=$.
If $\lambda$ is computable and mixable,
there exists a smallest, to within an additive constant,
upper semicomputable superloss process:
$$
  \exists L_1 \,
  \forall L_2 \,
  \exists c\in\mathbb{R} \,
  \forall\sigma\in\mathbf{Z}^*:
  L_1(\sigma)
  \le
  L_2(\sigma)
  +
  c,
$$
where $L_1$ and $L_2$ range over upper semicomputable superloss processes under $\lambda$.
(For a precise statement and proof,
see \cite{Kalnishkan:2015}, Theorem~1, Lemma~6, and Corollary~3;
% in the book: Theorem~8.1, Lemma~8.6, and Corollary~8.3;
\cite{Kalnishkan:2015} only considers the case of a trivial one-element $\mathbf{X}$,
but the extension to the case of general $\mathbf{X}$ is easy.)
For each computable mixable $\lambda$ (including the log, Brier, and spherical loss functions),
fix such a smallest upper semicomputable superloss process;
it will be denoted $\K^{\lambda}$, and $\K^{\lambda}(\sigma)$
will be called the \emph{predictive complexity} of $\sigma\in\mathbf{Z}^*$ under $\lambda$.
The intuition behind $\K^{\lambda}(\sigma)$ is that this is the loss
of the ideal prediction strategy whose computation is allowed
to take an infinite amount of time.
% From now on we only consider loss functions $\lambda$
% for which the conditions ensuring the existence of predictive complexity $\K^{\lambda}$
% are satisfied;
% the log and Brier loss functions are in this class.

In this note we consider infinite data sequences $\zeta\in\mathbf{Z}^{\infty}$,
which are idealizations of long finite data sequences.
If $\zeta=(z_1,z_2,\ldots)\in\mathbf{Z}^{\infty}$ and $T$ is a nonnegative integer,
we let $\zeta^T$ to stand for the prefix $z_1\ldots z_T$ of $\zeta$ of length $T$.

The \emph{randomness deficiency} of $\sigma\in\mathbf{Z}^*$
with respect to a prediction algorithm $F$ under a computable mixable loss function $\lambda$
is defined to be
\begin{equation}\label{eq:deficiency}
  D^{\lambda}_F(\sigma)
  :=
  \Loss_F^{\lambda}(\sigma)
  -
  \K^{\lambda}(\sigma);
\end{equation}
since $\Loss_F^{\lambda}$ is upper semicomputable
(\cite{Kalnishkan:2015}, Section~3.1),
% in the book: Section~8.3.1
the function $D^{\lambda}_F:\mathbf{Z}^*\to\mathbb{R}$ is bounded below.
Notice that the indeterminacy $\infty-\infty$ never arises in \eqref{eq:deficiency}
as $\K^{\lambda}<\infty$.
We will sometimes replace the upper index $\lambda$ in any of the three terms of \eqref{eq:deficiency}
by ``$\ln$'' % (respectively ``$\Brier$'')
in the case where $\lambda$ is the log loss function.
% (respectively the Brier loss function).

Let us say that $\zeta\in\mathbf{Z}^{\infty}$
is \emph{random with respect to $F$ under $\lambda$}
if 
$$
  \sup_T
  D^{\lambda}_F(\zeta^T)
  <
  \infty.
$$
The intuition is that in this case $F$ is an optimal prediction algorithm for $\zeta$
under~$\lambda$.

\ifFULL\bluebegin
  A sanity-check statement:
  \begin{proposition}
    $\zeta$ is log-random with respect to a prediction algorithm $F$
    for $F$-almost all $\zeta$
    (when $F$ is identified with the corresponding probability measure).
  \end{proposition}

  Two notions of randomness:
  universal (Martin-L\"of style) and non-universal (simplified Schnorr-style).
  The latter ($\zeta$ is said to be random with respect to a computable probability measure $P$
  if $Q/P$ is bounded on $\zeta$ for any computable probability measure $Q$)
  is applicable to all computable proper loss functions,
  whereas the former is only applicable to computable proper mixable loss functions.
  For Theorem~\ref{thm:simple} we need the former definition,
  since the supermartingale is required to be bounded below.
\blueend\fi

\subsection*{Log Randomness}

In the case where $\lambda$ is the log loss function and $\mathbf{X}$ is a one-element set,
the predictive complexity of a finite data sequence $\sigma$
(which is now a binary sequence if we ignore the uninformative objects)
is equal, to within an additive constant, to $-\ln M(\sigma)$,
where $M$ is Levin's \emph{a priori} semimeasure.
(In terms of this note, a \emph{semimeasure} can be defined as a process of the form $e^{-L}$
for some superloss process $L$ under the log loss function;
\emph{Levin's \emph{a priori} semimeasure} is a largest, to within a constant factor,
lower semicomputable semimeasure.)
The randomness deficiency $D^{\ln}_F(\sigma)$ of $\sigma$
with respect to a prediction algorithm $F$
is then, to within an additive constant, $\ln(M(\sigma)/P(\sigma))$,
where $P$ is the probability measure corresponding to~$F$,
$$
  P(y_1,\ldots,y_T)
  :=
  \bar p_1\cdots\bar p_T,
  \quad
  \bar p_t
  :=
  \begin{cases}
    F(y_1,\ldots,y_{t-1}) & \text{if $y_t=1$}\\
    1-F(y_1,\ldots,y_{t-1}) & \text{if $y_t=0$}
  \end{cases}
$$
(we continue to ignore the objects, which are not informative).
Therefore, $D^{\ln}_F(\sigma)$ is a version of the classical randomness deficiency of $\sigma$,
and $\zeta\in\{0,1\}^{\infty}$ is random with respect to $F$ under the log loss function
if and only if $\zeta$ is random with respect to $P$ in the sense of Martin-L\"of.

\section{A Simple Statement of Fundamentality}

In this section, we consider computable proper mixable loss functions.
\begin{theorem}\label{thm:simple}
  Let $\lambda$ be a CPM loss function.
  If a data sequence $\zeta\in\mathbf{Z}^{\infty}$ is random under the log loss function
  with respect to a prediction algorithm $F$,
  it is random under $\lambda$ with respect to~$F$.
\end{theorem}

A special case of this theorem is stated as Proposition~16 in \cite{Vovk:2001}.

Let us say that a CPM loss function $\lambda$ is \emph{fundamental}
if it can be used in place of the log loss function in Theorem~\ref{thm:simple}.
The proof of the theorem will in fact demonstrate its following quantitative form:
for any computable $\eta>0$ and any computable proper $\eta$-mixable $\lambda$
there exists a constant $c_{\lambda}$ such that,
for any prediction algorithm $F$,
\begin{equation}\label{eq:positive}
  D_F^{\ln}
  \ge
  \eta D_F^{\lambda} - c_{\lambda}.
\end{equation}
Let us define the \emph{mixability constant} $\eta_{\lambda}$
of a loss function $\lambda$
as the supremum of $\eta$ such that $\lambda$ is $\eta$-mixable.
It is known that a mixable loss function $\lambda$ is $\eta_{\lambda}$-mixable
(\cite{Vovk:1998game}, Lemmas~10 and~12);
therefore, \eqref{eq:positive} holds for $\eta=\eta_{\lambda}$,
provided $\eta_{\lambda}$ is computable.

If $\mathbf{X}$ is a one-element set
(and so the objects do not play any role and can be ignored),
the notion of randomness under the log loss function
coincides with the standard Martin-L\"of randomness,
as discussed in the previous section.
Theorem~\ref{thm:simple} shows that other notions of randomness
are either equivalent or weaker.

A \emph{superprediction} is a point in the plane
that lies Northeast of the prediction set \eqref{eq:prediction-set}
(i.e., a point $(x,y)\in\mathbb{R}^2$ such that $\lambda_0(p)\le x$ and $\lambda_1(p)\le y$
for some $p\in[0,1]$).

\begin{proof}[Proof of Theorem~\ref{thm:simple}]   % manual notCONF
% \begin{proof}[of Theorem~\ref{thm:simple}]       % manual CONF
  We will prove \eqref{eq:positive} for a fixed $\eta\in(0,\infty)$
  such that $\eta$ is computable and $\lambda$ is $\eta$-mixable.
  Let $L$ be a superloss process under $\lambda$ and $F$ be a prediction algorithm.
  Fix temporarily $(\sigma,x)\in\mathbf{Z}^*\times\mathbf{X}$ and set $p:=F(\sigma,x)\in[0,1]$;
  notice that $(a,b):=(L(\sigma,x,0)-L(\sigma),L(\sigma,x,1)-L(\sigma))$ is a $\lambda$-superprediction.
  By the definition of $\eta$-mixability there exists a parallel translation of the curve
  $e^{-\eta x}+e^{-\eta y}=1$
  that passes through the point $\lambda^p:=(\lambda_0(p),\lambda_1(p))$
  and lies Southeast of the prediction set of $\lambda$.
  Let $h$ be the affine transformation of the plane mapping that translation
  onto the curve $e^{-x}+e^{-y}=1$;
  notice that $h$ is the composition of the scaling $(x,y)\mapsto\eta(x,y)$ by $\eta$
  and then parallel translation moving the point $\eta\lambda^p$ to the point $(-\ln(1-p),-\ln p)$.
  The $\lambda$-superprediction $(a,b)$ is mapped by $h$ to the $\ln$-superprediction
  $$
    \left(
      \eta a + (-\ln(1-p)) - \eta\lambda_0(p),
      \eta b + (-\ln p) - \eta\lambda_1(p)
    \right).
  $$
  We can see that $\eta L+\Loss^{\ln}_F-\eta\Loss^{\lambda}_F$ is a superloss process under $\ln$.
  It is clear that this $\ln$-superloss process is upper semicomputable if $L$ is.
  Therefore, for some constant $c_{\lambda}$,
  $$
    \K^{\ln}
    \le
    \eta \K^{\lambda}+\Loss^{\ln}_F-\eta\Loss^{\lambda}_F
    +
    c_{\lambda},
  $$
  which is equivalent to~\eqref{eq:positive}.
  \ifCONFnotPURE\qed\fi
\end{proof}

\section{A Criterion of Fundamentality}

In this section, we only consider computable proper mixable loss functions
that satisfy, additionally, the smoothness conditions.
The main result of this section is the following elaboration of Theorem~\ref{thm:simple}
for CPMS loss functions.
\begin{theorem}\label{thm:criterion}
  A CPMS loss function $\lambda$ is fundamental if and only if
  \begin{equation}\label{eq:criterion-a}
    \inf_p
    (1-p)\lambda'_0(p)
    >
    0.
  \end{equation}
  Equivalently, it is fundamental if and only if
  \begin{equation}\label{eq:criterion-b}
    \inf_p
    (-p)
    \lambda'_1(p)
    >
    0.
  \end{equation}
\end{theorem}

We can classify CPMS loss functions $\lambda$ by their \emph{degree}
$$
  \deg(\lambda)
  :=
  \inf
  \left\{
    k:
    \lambda_0^{(k)}(0)\ne0 \text{ and } \lambda_1^{(k)}(1)\ne0
  \right\},
$$
where ${}^{(k)}$ stands for the $k$th derivative
and, as usual, $\inf\emptyset:=\infty$.
We will see later in this section that Theorem~\ref{thm:criterion}
can be restated to say that the fundamental loss functions are exactly those of degree~1.
Furthermore,
we will see that for a CPMS loss function $\lambda$ of degree $1<k<\infty$
there exist a data sequence $\zeta\in\mathbf{Z}^{\infty}$ and a prediction algorithm $F$
such that $\zeta$ is random with respect to $F$ under $\lambda$
while the randomness deficiency $D^{\ln}_F(\zeta^T)$
of $\zeta^T$ with respect to $F$ under the log loss function
grows almost as fast as $T^{1-1/k}$ as $T\to\infty$.

Straightforward calculations show that the log loss function has degree~1
and the Brier and spherical loss functions have degree~2.
% (The latter is symmetric and satisfies $\lambda''_0(p)=2$  % or is it 1?
% at $p=0$,
% similarly to the Brier loss function.)

In the proof of Theorem~\ref{thm:criterion} we will need the notion
of the \emph{signed curvature} of the \emph{prediction curve} $(\lambda_0(p),\lambda_1(p))$
at a point $p\in(0,1)$, which can be defined as
\begin{equation}\label{eq:curvature}
  k_{\lambda}(p)
  :=
  \frac{\lambda_0'(p)\lambda_1''(p)-\lambda_1'(p)\lambda_0''(p)}{(\lambda_0'(p)^2+\lambda_1'(p)^2)^{3/2}}.
\end{equation}
% Let $k_{\lambda}(p)$ be the curvature of the prediction set of the loss function $\lambda$
% at the point $(\lambda(p,0),\lambda(p,1))$.
The mixability constant $\eta_{\lambda}$
(i.e., the largest $\eta$ for which $\lambda$ is $\eta$-mixable) is
\begin{equation*}
  \eta_{\lambda}
  =
  \inf_p
  \frac{k_{\lambda}(p)}{k_{\ln}(p)}.
\end{equation*}
Therefore, $\lambda$ is mixable if and only if
\begin{equation}\label{eq:mixability-1}
  \inf_p
  \frac{k_{\lambda}(p)}{k_{\ln}(p)}
  >
  0.
\end{equation}

\begin{lemma}\label{lem:fundamentality}
  A CPMS loss function $\lambda$ is fundamental if and only if
  $$
    \sup_p
    \frac{k_{\lambda}(p)}{k_{\ln}(p)}
    <
    \infty
  $$
  (cf.\ \eqref{eq:mixability-1}).
\end{lemma}

The proof of the part ``if'' of Lemma~\ref{lem:fundamentality}
goes along the same lines as the proof of Theorem~\ref{thm:simple},
and also shows that,
if $\lambda$ and $\Lambda$ are CPMS loss functions
such that
$$
  \eta_{\lambda}
  :=
  \inf_p
  \frac{k_{\lambda}(p)}{k_{\ln}(p)}
  >
  0
  \text{\quad and\quad}
  H_{\Lambda}
  :=
  \sup_p
  \frac{k_{\Lambda}(p)}{k_{\ln}(p)}
  <
  \infty
$$
are computable numbers,
then there exists $c_{\lambda,\Lambda}\in\mathbb{R}$ such that,
for any prediction algorithm $F$,
\begin{equation*} % \label{eq:super-positive}
  H_{\Lambda} D_F^{\Lambda}
  \ge
  \eta_{\lambda} D_F^{\lambda} - c_{\lambda,\Lambda}.
\end{equation*}
We will call $H_{\Lambda}$ the \emph{fundamentality constant} of $\Lambda$
(analogously to $\eta_{\lambda}$ being called the mixability constant of $\lambda$).

Notice that the log loss function (perhaps scaled by multiplying by a positive constant)
is the only loss function for which the mixability and fundamentality constants coincide,
$\eta_{\ln}=H_{\ln}$.
Therefore, fundamental CPMS loss functions can be regarded as log-loss-like.

The part ``only if'' of Lemma~\ref{lem:fundamentality} will be proved below,
in the proof of Theorem~\ref{thm:criterion}.

The computation of $k_{\lambda}$ for the three basic loss functions using~\eqref{eq:curvature} gives:
\begin{itemize}
\item
  For the log loss function, the result is
  \begin{equation}\label{eq:log-curvature}
    k_{\ln}(p)
    =
    \frac{p(1-p)}{(p^2+(1-p)^2)^{3/2}}.
  \end{equation}
  \ifFULL\bluebegin
    Computations:
    $$
      \lambda_0(p)=-\ln(1-p),
      \lambda_1(p)=-\ln p,
      \lambda_0'=\frac{1}{1-p},
      \lambda_1'=-\frac{1}{p},
      \lambda_0''=\frac{1}{(1-p)^2},
      \lambda_1''=\frac{1}{p^2}.
    $$
    Therefore,
    \begin{multline*}
      k
      :=
      \frac{\lambda_0'\lambda_1''-\lambda_1'\lambda_0''}{(\lambda_0'^2+\lambda_1'^2)^{3/2}}
      =
      \frac
        {\frac{1}{1-p}\frac{1}{p^2}+\frac{1}{p}\frac{1}{(1-p)^2}}
        {\left(\frac{1}{(1-p)^2}+\frac{1}{p^2}\right)^{3/2}}
      =
      \frac
        {p^3(1-p)^3\left(\frac{1}{1-p}\frac{1}{p^2}+\frac{1}{p}\frac{1}{(1-p)^2}\right)}
        {\left(p^2(1-p)^2\left(\frac{1}{(1-p)^2}+\frac{1}{p^2}\right)\right)^{3/2}}\\
      =
      \frac
        {(1-p)^2p+p^2(1-p)}
        {\left(p^2+(1-p)^2\right)^{3/2}}
      =
      \frac
        {p(1-p)}
        {\left(p^2+(1-p)^2\right)^{3/2}}.
    \end{multline*}
  \blueend\fi
\item
  For the Brier loss function, the result is
  $$
    k_{\text{Brier}}(p)
    =
    \frac12 \frac{1}{(p^2+(1-p)^2)^{3/2}}.
  $$
  \ifFULL\bluebegin
    Computations:
    $$
      x=p^2,
      y=(1-p)^2,
      \lambda_0'=2p,
      \lambda_1'=-2(1-p),
      \lambda_0''=2,
      \lambda_1''=2.
    $$
    Therefore,
    $$
      k
      :=
      \frac{\lambda_0'\lambda_1''-\lambda_1'\lambda_0''}{(\lambda_0'^2+\lambda_1'^2)^{3/2}}
      =
      \frac
        {4p+4(1-p)}
        {\left(4p^2+4(1-p)^2\right)^{3/2}}
      =
      \frac12
      \frac
        {1}
        {\left(p^2+(1-p)^2\right)^{3/2}}.
    $$
  \blueend\fi
\item
  For the spherical loss function, the result is
  $$
    k_{\text{spher}}(p)
    =
    1.
  $$
\end{itemize}

We can plug the expression~\eqref{eq:log-curvature} for the signed curvature of the log loss function
into Lemma~\ref{lem:fundamentality} to obtain a more explicit statement.
Because of the propriety of $\lambda$, this statement can be simplified,
which gives the following corollary.
\begin{corollary}\label{cor:fundamentality-1}
  A CPMS loss function $\lambda$ is fundamental if and only if
  \begin{equation}\label{eq:fundamentality-1}
    \sup_p
    \frac
      {\lambda_0'(p)\lambda_1''(p)-\lambda_1'(p)\lambda_0''(p)}
      {\lambda_0'(p)\lambda_1'(p)(\lambda_1'(p)-\lambda_0'(p))}
    <
    \infty.
  \end{equation}
\end{corollary}
\begin{proof}
  In view of the expressions~\eqref{eq:curvature} and~\eqref{eq:log-curvature},
  the condition in Lemma~\ref{lem:fundamentality} can be written as
  \begin{equation*}
    \sup_p
    \frac
      {\lambda_0'(p)\lambda_1''(p)-\lambda_1'(p)\lambda_0''(p)}
      {(\lambda_0'(p)^2+\lambda_1'(p)^2)^{3/2}}
    \frac{(p^2+(1-p)^2)^{3/2}}{p(1-p)}
    <
    \infty.
  \end{equation*}
  Therefore, it suffices to check that
  \begin{equation*}
    \frac{(\lambda_0'(p)^2+\lambda_1'(p)^2)^{3/2}}{\lambda_0'(p)\lambda_1'(p)(\lambda_1'(p)-\lambda_0'(p))}
    =
    \frac{(p^2+(1-p)^2)^{3/2}}{p(1-p)}.
  \end{equation*}
  The last equality follows from
  \begin{equation}\label{eq:proper-property}
    \frac{\lambda_1'(p)}{\lambda_0'(p)} = \frac{p-1}{p},
  \end{equation}
  which in turn follows from the propriety of $\lambda$.
  \ifCONFnotPURE\qed\fi
\end{proof}

It is instructive to compare the criterion~\eqref{eq:fundamentality-1}
with the well-known criterion
\begin{equation}\label{eq:mixability-2}
  \inf_p
  \frac
    {\lambda_0'(p)\lambda_1''(p)-\lambda_1'(p)\lambda_0''(p)}
    {\lambda_0'(p)\lambda_1'(p)(\lambda_1'(p)-\lambda_0'(p))}
  >
  0
\end{equation}
for $\lambda$ being mixable
(see, e.g., \cite{Haussler/etal:1998} or \cite{Kalnishkan/Vovk:2000}, Theorem~2;
it goes back to \cite{Vovk:1990}, Lemma~1).
The criterion~\eqref{eq:mixability-2} can be derived from \eqref{eq:mixability-1}
as in the proof of Corollary~\ref{cor:fundamentality-1}.

\begin{proof}[Proof of Theorem~\ref{thm:criterion}]   % manual notCONF
% \begin{proof}[of Theorem~\ref{thm:criterion}]       % manual CONF
  Differentiating~\eqref{eq:proper-property} we obtain
  $$
    \frac{\lambda_1''(p)\lambda_0'(p)-\lambda_1'(p)\lambda_0''(p)}{\lambda_0'(p)^2}
    =
    p^{-2},
  $$
  and the fundamentality constant~\eqref{eq:fundamentality-1} of $\lambda$ is
  \begin{multline*}
    \sup_p \frac{p^{-2}\lambda_0'(p)^2}{\lambda_0'(p)\lambda_1'(p)(\lambda_1'(p)-\lambda_0'(p))}
    =
    \sup_p \frac{p^{-2}}{\lambda_0'(p)(\lambda_1'(p)/\lambda_0'(p))(\lambda_1'(p)/\lambda_0'(p)-1)}\\
    =
    \sup_p \frac{p^{-2}}{\lambda_0'(p)(1-1/p)(-1/p)}
    =
    \sup_p \frac{1}{\lambda_0'(p)(1-p)},
  \end{multline*}
  where we have used \eqref{eq:proper-property}.
  This gives us~\eqref{eq:criterion-a};
  in combination with~\eqref{eq:proper-property}
  we get~\eqref{eq:criterion-b}.

  Let us now prove the part ``only if'' of Theorem~\ref{thm:criterion}
  (partly following the argument given after Proposition~16 of \cite{Vovk:2001}).
  According to~\eqref{eq:proper-property},
  \eqref{eq:criterion-a} and \eqref{eq:criterion-b} are equivalent.
  Suppose that
  \begin{equation*} % \label
    \inf_p
    (1-p)\lambda'_0(p)
    =
    0,
  \end{equation*}
  and let us check that $\lambda$ is not fundamental.
  By the smoothness assumptions,
  we have $(1-p)\lambda'_0(p)=0$ either for $p=0$ or for $p=1$.
  Suppose, for concreteness, that $(1-p)\lambda'_0(p)=0$ for $p=0$
  (if $(1-p)\lambda'_0(p)=0$ for $p=1$, we will have 
  $(-p)\lambda'_1(p)=0$ for $p=1$,
  and we can apply the same argument as below for $p=1$ in place of $p=0$).
  Let $k$ be such that $\lambda^{(k)}_0(0)>0$ but $\lambda^{(i)}_0(0)=0$ for all $i<k$;
  we know that $k\ge2$
  (the easy case where $\lambda^{(i)}_0(0)=0$ for all $i$ should be considered separately).
  Consider
  % a vacuous (i.e., one-element) $\mathbf{X}$,
  % so that $\mathbf{Z}$ can be identified with $\{0,1\}$,
  any data sequence $\zeta=(x_1,y_1,x_2,y_2,\ldots)\in\mathbf{Z}^{\infty}$
  in which all labels are 0: $y_1=y_2=\cdots=0$.
  We then have $\sup_T\K^{\ln}(\zeta^T)<\infty$ and $\sup_T\K^{\lambda}(\zeta^T)<\infty$.
  % The predictive complexity of this sequence is $0$.
  Let $F$ be the prediction algorithm that outputs $p_t:=t^{-1/k-\epsilon}$ at step $t$,
  where $\epsilon\in(0,1-1/k)$.
  Then $\zeta$ is random with respect to $F$ under $\lambda$
  since the loss of this prediction algorithm over the first $T$ steps is
  $$
    \sum_{t=1}^T
    \lambda_0(p_t)
    \le
    2
    \sum_{t=1}^T
    \frac{\lambda^{(k)}_0(0)}{k!}
     p_t^k
     +
     O(1)
  $$
  (we have used Taylor's approximation for $\lambda_0$)
  and the series $\sum_t p_t^k$ is convergent.
  % (by Theorem~23 in \cite{Vovk:2001} and the fact that the series $\sum_t p_t^k$ is convergent).
  On the other hand, the randomness deficiency of $\zeta^T$ with respect to $F$
  under the log loss function grows as
  \begin{equation*}
    -\sum_{t=1}^T
    \ln(1-p_t)
    \sim
    \sum_{t=1}^T
    p_t
    \sim
    \frac{k}{k-1-k\epsilon}
    T^{1-1/k-\epsilon}.
    % \tag*{$\qed$}  % used to give a large vertical gap after the proof
    \qedhere   % manual
  \end{equation*}
  % \ifCONF\hfill\qed\fi
  % \renewcommand{\qedsymbol}{}
  % See the explanation in IVAP.tex.
\end{proof}

Notice that the criterion of mixability~\eqref{eq:mixability-2} can be simplified
when we use~\eqref{eq:proper-property}:
it becomes
\begin{equation*}
  \sup_p
  (1-p)\lambda_0'(p)
  <
  \infty
\end{equation*}
or, equivalently,
\begin{equation*}
  \sup_p
  (-p)\lambda_1'(p)
  <
  \infty.
\end{equation*}
The function
$
  (1-p)\lambda_0'(p)
  =
  (-p)\lambda_1'(p)
$
can be computed as
\begin{itemize}
\item
  $1$ in the case of the log loss function;
\item
  $2p(1-p)$ in the case of the Brier loss function;
\item
  $p(1-p)(p^2+(1-p)^2)^{-3/2}$ in the case of the spherical loss function.
\end{itemize}
Therefore, all three loss functions are mixable,
but only the log loss function is fundamental.

It is common in experimental machine learning to truncate allowed probabilistic predictions
to the interval $[\epsilon,1-\epsilon]$ for a small constant $\epsilon>0$
(this boils down to cutting off the ends of the prediction sets
corresponding to the slopes below $\epsilon$ and above $1-\epsilon$).
It is easy to check that in this case all CPMS loss functions lead to the same notion of randomness.
\begin{corollary}
  CPMS loss functions $\lambda$ and $\Lambda$ restricted to $p\in[\epsilon,1-\epsilon]$,
  where $\epsilon>0$, lead to the same notion of randomness.
\end{corollary}
We can make the corollary more precise as follows:
for prediction algorithms $F$ restricted to $[\epsilon,1-\epsilon]$,
$D^{\lambda}_F$ and $D^{\Lambda}_F$ coincide to within a factor of
$$
  \max
  \left(
    \sup_{p\in[\epsilon,1-\epsilon]}
    \frac{k_{\lambda}(p)}{k_{\Lambda}(p)},
    \sup_{p\in[\epsilon,1-\epsilon]}
    \frac{k_{\Lambda}(p)}{k_{\lambda}(p)}
  \right)
$$
and an additive constant.

\section{Frequently Asked Questions}
\label{sec:FAQ}

This section is more discursive than the previous ones;
``frequently'' in its title means ``at least once''
(but with a reasonable expectation that a typical reader might well ask similar questions).
% \begin{description}[style=nextline]
\begin{Question}
  What is the role of the requirement of propriety in Theorem~\ref{thm:simple}?
\end{Question}
The theorem says that the log loss function leads to the most restrictive notion of randomness:
if a sequence is random with respect to some prediction algorithm under the log loss function,
then it is random with respect to the ``same'' prediction algorithm under an arbitrary CPM loss function.
One should explain, however, what is meant by the same prediction algorithm,
because of the freedom in parameterization
(say, we can replace each prediction $p$ by $p^2$).
The requirement of propriety imposes a canonical parameterization.
\begin{Question}
  What is the role of the requirement of mixability in Theorem~\ref{thm:simple}?
\end{Question}
The requirement of mixability ensures the existence of predictive complexity,
which is used in the definition of predictive randomness.
\begin{Question}
  Mixability is sufficient for the existence of predictive complexity
  (for computable loss functions).
  Is it also necessary?
\end{Question}
Yes, it is: see Theorem~1 in \cite{Kalnishkan/etal:2004ALT}.
\begin{Question}
  What is the geometric intuition behind the notions of propriety and mixability?
\end{Question}
  The intuitions behind the two notions overlap;
  both involve requirements of convexity of the ``superprediction set''
  (the area Northeast of the prediction set \eqref{eq:prediction-set}).
  Let us suppose that the loss function $\lambda$ is continuous in the prediction $p$,
  so that the prediction set is a curve.
  % Wikipedia, article "curve": the requirement of continuity is part of the definition.
  Propriety then means that the superprediction set is strictly convex
  (in particular, the prediction set has no straight segments)
  and that the points on the prediction set are indexed in a canonical way
  (namely, each such point is indexed by $1/(1-s)$
  where $s<0$ is the slope of the tangent line to the prediction set at that point:
  cf.\ \eqref{eq:proper-property}).
  Mixability means that the superprediction set is convex in a stronger sense:
  it stays convex after being transformed by the mapping
  $(x,y)\in[0,\infty]^2\mapsto(e^{-\eta x},e^{-\eta y})$ for some $\eta>0$.
\begin{Question}
  Why should we consider not only the log loss function
  (which nicely corresponds to probability distributions) but also other loss functions?
  You say ``the log loss function, being most selective,
  should be preferred to the alternatives such as Brier or spherical loss''.
  But this does not explain why these other loss functions were interesting in the first place.
\end{Question}
Loss functions different from the log loss function are widely used in practice;
in particular, the Brier loss function is at least as popular as
(and perhaps even more popular than) the log loss function in machine learning:
% and meteorology
see, e.g., the extensive empirical study \cite{Caruana/NM:2006}.
An important reason for the popularity of Brier loss
is that the log loss function often leads to infinite average losses on large test sets
for state-of-the-art prediction algorithms, which is considered to be ``unfair'',
and some researchers even believe that any reasonable loss function should be bounded.
% I have heard this from John Langford
% although can't find any references

\ifFULL\bluebegin
  \section{Open problems}

  \begin{enumerate}
  \item
    Prove that \eqref{eq:positive} cannot be improved (the coefficient $\eta$ is optimal).
    Prove a similar statement for~\eqref{eq:super-positive}
    (the two constants $\eta_{\lambda}$ and $H_{\Lambda}$ now, however,
    should be merged into a single constant to achieve optimality).
  \item
    Give the best possible quantitative version of Theorem~\ref{thm:criterion}.
    In particular, is the growth rate about $T^{1/2}$ optimal for the Brier loss function
    (and other degree 2 loss functions)?
    Show that the rate is about $T^{1-1/k}$ for CPMS loss functions of degree $k$.
  \item
    Get rid of the assumption of mixability replacing Martin-L\"of randomness
    by Schnorr randomness.
  \item
    Get rid of the assumption of smoothness of the loss functions.
  \end{enumerate}
\blueend\fi

\section{Conclusion}

This note offers an answer to the problem of choosing a loss function
for evaluating probabilistic prediction algorithms
in experimental machine learning.
Our answer is that the log loss function, being most selective,
should be preferred to the alternatives such as Brier or spherical loss.
\ifCONF
  This answer, however, remains asymptotic (involving unspecified constants)
  and raises further questions.
  To make it really practical, we need to restrict our generalized theory of algorithmic randomness,
  as Yuri did in a different context.
\fi

\subsubsection*{Acknowledgments.}
I am grateful to Mitya Adamskiy, Yuri Kalnishkan, Ilia Nouretdinov, Ivan Petej, and Vladimir V'yugin
for useful discussions.
Thanks to an anonymous reviewer of the conference version of this note
whose remarks prompted me to add Section~\ref{sec:FAQ}
(and were used in both questions and answers).
This work has been supported by EPSRC (grant EP/K033344/1)
and the Air Force Office of Scientific Research (grant ``Semantic Completions'').

\end{document}